\documentclass{article}
\usepackage{spconf,amsmath,graphicx,hyperref}
\hypersetup{hidelinks}
\usepackage{fancyhdr}
\fancypagestyle{arxivfirst}{
  \fancyhf{}

  \fancyfoot[C]{\parbox{\textwidth}{\footnotesize\raggedright This work has been submitted to the IEEE for possible publication. Copyright may be transferred without notice, after which this version may no longer be accessible.}}
}

\usepackage[utf8]{inputenc}
\usepackage{booktabs}
\usepackage{amsfonts,amssymb,amsthm}
\usepackage[ruled,vlined,linesnumbered]{algorithm2e}
\SetKwComment{Comment}{$\triangleright$ }{}
\usepackage{multirow}
\usepackage{xspace}
\usepackage{microtype}
\usepackage{nicefrac}
\usepackage{enumitem}
\usepackage{balance}
\allowdisplaybreaks
\newtheorem{theorem}{Theorem}
\newtheorem{lemma}[theorem]{Lemma}
\newtheorem{proposition}[theorem]{Proposition}
\newtheorem{corollary}[theorem]{Corollary}
\newtheorem{definition}{Definition}

\newtheorem{remark}{Remark}

\newcommand{\bx}{\mathbf{x}}
\newcommand{\by}{\mathbf{y}}
\newcommand{\bu}{\mathbf{u}}
\newcommand{\bm}{\mathbf{m}}
\newcommand{\bI}{\mathbf{I}}
\newcommand{\bK}{\mathbf{K}}
\newcommand{\bk}{\mathbf{k}}

\newcommand{\bL}{\mathbf{L}}

\newcommand{\bw}{\mathbf{w}}
\newcommand{\bs}{\mathbf{s}}

\newcommand{\calD}{\mathcal{D}}
\newcommand{\calX}{\mathcal{X}}

\newcommand{\calN}{\mathcal{N}}

\DeclareMathOperator*{\argmax}{\mathrm{arg\,max}}

\newcommand{\qexpd}{\ensuremath{q}\text{-ED}\xspace}
\newcommand{\qED}[3]{q\text{-ED}_{#1}(#2,#3)}
\newcommand{\qexpdbo}{\ensuremath{q}\text{-ED-BO}\xspace}
\newcommand{\qei}{\ensuremath{q}\text{-EI}\xspace}
\newcommand{\qucb}{\ensuremath{q}\text{-UCB}\xspace}

\title{Robust Bayesian Optimization with Q-Exponential Surrogates}
\name{
Richard Cornelius Suwandi$^{1}$, 
Zhidi Lin$^{2}$, 
Feng Yin$^{1}$,
Abdelhak M. Zoubir$^{3}$%
\thanks{Feng Yin (\textit{yinfeng@cuhk.edu.cn}) is the corresponding author.}
}
\address{$^{1}$School of Artificial Intelligence, The Chinese University of Hong Kong, Shenzhen, China \\
$^{2}$School of Computing and Data Science, The University of Hong Kong, Hong Kong SAR, China \\
$^{3}$Signal
Processing Group, Technische Universit\"{a}t Darmstadt, Germany}

\begin{document}
\ninept
\maketitle
\thispagestyle{arxivfirst}
\begin{abstract}
Bayesian optimization (BO) is a widely used framework for optimizing expensive black-box objectives, but standard BO methods often use Gaussian process (GP) surrogates whose Gaussian assumption is sensitive to outliers and heavy-tailed noise.
We introduce \qexpdbo{}, a robust BO method whose surrogate follows a univariate $q$-exponential (\qexpd) distribution, preserving GP-BO's closed-form posterior mean and variance while a shape parameter $q$ controls the tail behavior, recovering the GP at $q{=}2$ and growing heavier-tailed with wider confidence bounds as $q$ decreases.
This tractability yields a closed-form $q$-upper confidence bound (\qucb) with sublinear regret, and an exact closed-form $q$-expected improvement (\qei) that generalizes EI to the heavy-tailed predictive, recovering classical EI at $q{=}2$.
Experiments on beamformer and adaptive filter tuning with impulsive outliers show that \qexpdbo{} matches or exceeds existing baselines on clean data, and under corruption, improves the strongest baseline by approximately $0.7\,\mathrm{dB}$ in output SINR and $1.1$--$1.2\,\mathrm{dB}$ in misalignment reduction.
\end{abstract}

\section{Introduction}
\label{sec:intro}

Bayesian optimization (BO) seeks the optimum of an expensive black-box objective $f$ using as few evaluations as possible. It uses a probabilistic surrogate model and an acquisition function to guide the search \cite{garnett2023bayesian}. Gaussian processes (GPs) are commonly used as surrogates because of their analytical tractability and principled uncertainty quantification \cite{rasmussen2006gaussian}, and have proven attractive in signal processing applications ranging from dynamical-system modeling \cite{yin2020linear,lin2024ensemble} to distributed learning \cite{suwandi2025gsmp,suwandi2022gsm}. However, the robustness of GP-based BO rests on its Gaussian likelihood: because both the prior and the noise models are Gaussian, a single corrupted observation can shift the posterior mean and inflate predictive uncertainty \cite{zoubir2012robust}.

Several robust BO methods address this limitation by introducing heavy tails either into the likelihood or into the prior. Outlier-filtering BO uses a GP with a Student-$t$ likelihood to identify anomalous observations before applying standard expected improvement (EI) \cite{jylanki2011robust}, but the GP--Student-$t$ pair is non-conjugate, so the posterior requires an approximation and the predictive tails driving the acquisition are not available in closed form. Student-$t$-process (TP) BO instead employs a heavy-tailed function prior with exact predictive distributions \cite{shah2014student}, but its polynomial tails lack the exponential concentration needed for regret analysis. A third line, heavy-tailed-payoff BO, truncates rewards to handle heavy-tailed observation noise \cite{chowdhury2019heavytailed}, and a fourth considers distributionally and adversarially robust BO \cite{kirschner2020distributionally,bogunovic2018adversarially}, which hedges against input-distribution shifts or adversarial perturbations rather than corrupted evaluations of $f$ and is complementary to the surrogate-level robustness pursued here.

These approaches leave two gaps. TP BO has no regret guarantee because its polynomial tails cannot yield analytically calibrated confidence widths, and the only existing guarantee \cite{chowdhury2019heavytailed} covers heavy-tailed observation noise rather than a corrupted objective function. In the latter setting, the heavy tail must belong in the surrogate's predictive rather than its likelihood, since a heavy-tailed likelihood only robustifies how observations are fit, while the predictive that drives exploration stays as heavy-tailed as the prior.

We close these gaps with a \qexpd{} surrogate built from a generalized power-exponential family \cite{gomez1998multivariate}, where the shape parameter $q$ controls the tail behavior and the robustness-exploration trade-off. For $q<2$, the univariate predictive becomes progressively heavier-tailed, while at $q{=}2$ it recovers the GP surrogate.
Our contributions are as follows:
\begin{itemize}[leftmargin=*,itemsep=2pt,topsep=2pt]
\item We formulate \qexpdbo{}, whose surrogate predictive is a \qexpd{} distribution parameterized by the GP's closed-form posterior mean and variance, with $q$ controlling tail heaviness (\S~\ref{sec:surrogate}).
\item We prove a closed-form tail bound for \qexpd{} predictives and establish $R_T=\tilde{O}\bigl((\log T)^{1/q}\sqrt{T\gamma_T}\bigr)$ for \qucb{}. The result recovers GP-UCB at $q{=}2$. For $q<2$, heavier predictive tails cost only the factor $(\log T)^{1/q-1/2}$, preserving the leading $\sqrt{T\gamma_T}$ dependence and no-regret learning. We also derive exact \qei{} in terms of the Gaussian CDF and upper incomplete gamma function (\S~\ref{sec:acquisition}).%
\item We evaluate \qexpdbo{} on beamformer and adaptive filter tuning with outliers, best at $q{=}1$ for \qei{} on clean data and, under corruption, improving the strongest baseline by ${\approx}0.7\,\mathrm{dB}$ in SINR and $1.1$--$1.2\,\mathrm{dB}$ in misalignment reduction (\S~\ref{sec:experiments}).
\end{itemize}

The rest of the paper is organized as follows. Section~\ref{sec:background} reviews BO and the \qexpd{} family. Section~\ref{sec:method} introduces the \qexpd{} surrogate and derives the \qucb{} and \qei{} acquisitions along with the regret bound. Section~\ref{sec:experiments} evaluates \qexpdbo{} on beamformer and adaptive filter tuning with impulsive outliers, and Section~\ref{sec:conclusion} concludes the paper.

\section{Background}
\label{sec:background}

\subsection{Bayesian Optimization (BO)}
\label{sec:bo_prelim}

Bayesian optimization (BO) sequentially explores a black-box objective
$f$ by fitting a probabilistic surrogate to the observations seen so far
and using it to decide where to query next \cite{garnett2023bayesian}. Given $t$ noisy observations
$\calD_t=\{(\bx_i,y_i)\}_{i=1}^t$ satisfying $y_i=f(\bx_i)+\varepsilon_i$ with
$\varepsilon_i\stackrel{\mathrm{iid}}{\sim}\calN(0,\sigma_\varepsilon^2)$, the
objective $f(\bx)$ is often modeled with a Gaussian process (GP) prior \cite{rasmussen2006gaussian}. Under this Gaussian prior and Gaussian noise, the posterior mean and
variance at $\bx$ are
\begin{align}
  \mu_t(\bx) &= \bk_t(\bx)^\top(\bK_t+\sigma_\varepsilon^2\bI)^{-1}\by,
  \label{eq:gp_mean}\\
  \sigma_t^2(\bx) &= k(\bx,\bx)
    - \bk_t(\bx)^\top(\bK_t+\sigma_\varepsilon^2\bI)^{-1}\bk_t(\bx),
  \label{eq:gp_var}
\end{align}
where $\bK_t=[k(\bx_i,\bx_j)]_{t\times t}$ is the Gram matrix and
$\bk_t(\bx)$ is the $t\times 1$ vector of kernel evaluations against $\calD_t$. An acquisition function then uses this posterior to select the next query, trading off exploitation of $\mu_t(\bx)$ against exploration of $\sigma_t(\bx)$. Two widely used choices are expected improvement (EI) \cite{jones1998efficient} and the upper confidence bound (UCB) \cite{srinivas2012gaussian}.

\subsection{The \texorpdfstring{$q$}{q}-Exponential Family (\qexpd)}
\label{sec:qexpd_prelim}

We review the multivariate power-exponential distribution
\cite{gomez1998multivariate}, which supplies the heavy-tailed family used as our surrogate in \S~\ref{sec:method}. With $r\equiv r(\bu)=(\bu-\bm)^\top\bK^{-1}(\bu-\bm)$, the
$n$-dimensional $q$-exponential distribution $\qED{n}{\bm}{\bK}$ has the form
\begin{equation}
  p(\bu\mid\bm,\bK,q) = \tfrac{q}{2}(2\pi)^{-\frac n2}|\bK|^{-\frac12}
  r^{(\frac q2-1)\frac n2}\exp\!\left\{-\tfrac12 r^{q/2}\right\}.
  \label{eq:qed}
\end{equation}
For $0<q<2$ the density is more peaked at the mode and heavier-tailed, and it reduces to $\calN_n(\bm,\bK)$ at $q{=}2$. When $n{=}1$, $r=((u-m)/\sigma)^2$.

\begin{proposition}[Representation \cite{gomez1998multivariate}]
\label{prop:stoch_rep}
Every $\qED{n}{\bm}{\bK}$ random vector admits the stochastic representation $\bu=\bm+R\,\bL\bs$, with $\bL$ being a Cholesky factor of $\bK$, $\bs\sim\mathrm{Unif}(\mathbb{S}^{n-1})$ uniform on the unit sphere $\mathbb{S}^{n-1}\subset\mathbb{R}^n$ (for $n{=}1$, this reduces to $\bs\sim\mathrm{Unif}\{-1,+1\}$), and $R\ge0$ a scalar radial variable independent of $\bs$ whose $q$-th power satisfies $R^{q}\sim\chi^2_n$.
\end{proposition}

\begin{proposition}[Variance scaling \cite{gomez1998multivariate}]
\label{prop:scaling}
If $\bu\sim\qED{n}{\bm}{\bK}$, then $\mathrm{E}[\bu] = \bm$ and
$\mathrm{Cov}(\bu) = c(q,n)\,\bK$, where $\Gamma(\cdot)$ is the Gamma function and
\begin{equation}
c(q,n) = \frac{2^{2/q}\,\Gamma\left(\frac{n}{2}+\frac{2}{q}\right)}{n\,\Gamma\left(\frac{n}{2}\right)}\,.
\label{eq:cqn}
\end{equation}
\end{proposition}
\noindent The second argument of $\qED{n}{\cdot}{\cdot}$ is therefore not the
covariance, except at $q{=}2$ where $c(2,n)=1$.

\section{\qexpd{} Bayesian Optimization (\qexpdbo)}
\label{sec:method}

\qexpdbo{} consists of a \qexpd{} surrogate and an acquisition that selects the next query: \S~\ref{sec:surrogate} specifies the predictive and its concentration bound, \S~\ref{sec:acquisition} derives the acquisition, and Algorithm~\ref{alg:qexpdbo} summarizes the loop. 

\subsection{The \qexpd{} Surrogate}
\label{sec:surrogate}
\label{sec:tailbound}

At each candidate $\bx$, we take the predictive to be the $n{=}1$ case of
Eq.~\eqref{eq:qed}, with parameters $\mu_t(\bx)$ and $\sigma_t^2(\bx)$
from Eqs.~\eqref{eq:gp_mean}--\eqref{eq:gp_var}. This keeps the
posterior mean and variance in closed form while $q$ sets how heavy
the tails are.

\begin{definition}[Univariate \qexpd{} predictive]
\label{def:predictive}
Given $\calD_t$, the predictive at a test point $\bx$ is
\begin{equation}
  f(\bx)\mid\calD_t \;\sim\; \qED{1}{\mu_t(\bx)}{\sigma_t^2(\bx)},
  \label{eq:qexpdbo_post}
\end{equation}
where $\mu_t(\bx)$ and $\sigma_t^2(\bx)$ are the posterior mean and
variance in Eqs.~\eqref{eq:gp_mean}--\eqref{eq:gp_var}. Its mean is
$\mu_t(\bx)$ and its variance is $c(q)\,\sigma_t^2(\bx)$, with
\begin{equation}
  c(q):=c(q,1)=2^{2/q}\,\Gamma(\tfrac12+\tfrac2q)/\Gamma(\tfrac12)
  \label{eq:cq}
\end{equation}
for the $n{=}1$ case of Proposition~\ref{prop:scaling}. For $q<2$,
$c(q)>1$ and the predictive is heavy-tailed, while at $q{=}2$, $c(2)=1$ and
Eq.~\eqref{eq:qexpdbo_post} is Gaussian.
\end{definition}

Unlike polynomial-tailed alternatives, this predictive admits a closed-form
stretched-exponential tail bound, so a UCB
width can be written explicitly in $q$.

\begin{lemma}[Closed-form tail bound]
\label{lem:tail}
Let $U\sim\qED{1}{m}{\sigma^2}$ as in
Definition~\ref{def:predictive}. Then for all $t\ge0$,
\begin{equation}
  \Pr\bigl(|U-m| > t\,\sigma\bigr) \;\le\; \exp\!\left(-\tfrac12 t^{q}\right).
  \label{eq:tailbound}
\end{equation}
\end{lemma}
\begin{proof}
By Proposition~\ref{prop:stoch_rep} with $n=1$, $U=m+R\sigma S$ for
$S\sim\mathrm{Unif}\{-1,+1\}$, $R\ge0$, $R\perp S$, and $R^q\sim\chi^2_1$. Since
$\chi^2_1\stackrel{d}{=}N^2$ for $N\sim\calN(0,1)$, we have
$R\stackrel{d}{=}|N|^{2/q}$, so for $t\ge0$,
$\Pr(R>t)=\Pr(|N|^{2/q}>t)=\Pr(|N|>t^{q/2})\le\exp(-\tfrac12 t^{q})$,
using $\Pr(|N|>a)\le e^{-a^2/2}$ for all $a\ge0$. Since
$|U-m|=R\sigma$, this is Eq.~\eqref{eq:tailbound}.
\end{proof}

The tail is a stretched exponential, so the \qexpd{}
predictive belongs to the sub-Weibull family \cite{vladimirova2020subweibull}. Inverting
Eq.~\eqref{eq:tailbound} at level $\delta_t/|\calX|$ produces the
width $\rho=\bigl(2\log(|\calX|/\delta_t)\bigr)^{1/q}$, which reduces
to $\sqrt{2\log(|\calX|/\delta_t)}$ at $q{=}2$ and grows as $q$ decreases, so a smaller $q$ yields heavier tails at the cost of a wider interval. Following the same union-bound device used for GP-UCB \cite{srinivas2012gaussian}, we pick the per-round budget $\delta_t=6\delta/(\pi^2 t^2)$ so that $\sum_{t\ge1}\delta_t=\delta$, using $\sum_{t\ge1}1/t^2=\pi^2/6$. This schedule yields Eq.~\eqref{eq:rho} and a union bound of $\delta$ over all $\bx\in\calX$ and $t\ge1$.

\begin{remark}[A generalized-Bayes view of the reshaping]
\label{rem:genbayes}
Reparameterizing the GP posterior as a \qexpd{} follows from generalized Bayesian updating \cite{bissiri2016general} and tempered posteriors \cite{bhattacharya2019fractional}. Concurrent work inflates the noise variance \cite{li2026robust}, widening the predictive but leaving its Gaussian tail unchanged. The rescaling $c(q)$ in Eq.~\eqref{eq:cq} similarly widens, but the exponent $r^{q/2}$ in Eq.~\eqref{eq:qed} reshapes the tail, yielding Lemma~\ref{lem:tail}'s sub-Weibull decay, so $q$ controls the shape of the tail rather than only its width.
\end{remark}

\subsection{The \qexpd{} Acquisition}
\label{sec:acquisition}
\label{sec:regret}

The tail bound of \S~\ref{sec:tailbound} turns directly into a
closed-form acquisition rule for choosing $\bx_t$.

\subsubsection{\texorpdfstring{$q$}{q}-Upper Confidence Bound (\qucb)}
\label{sec:qucb}

We follow the finite-domain setting of GP-UCB
\cite{srinivas2012gaussian}, which our method recovers when
$q{=}2$. We assume that $\calX$ is finite (\textbf{A1}), extending
to a compact $\calX$ by discretization \cite{srinivas2012gaussian},
and that
Lemma~\ref{lem:tail} supplies valid tail
probabilities for $f(\bx)$ under Eq.~\eqref{eq:qexpdbo_post}
(\textbf{A2}). The kernel $k$ and the
noise variance $\sigma_\varepsilon^2$ are held fixed. The maximum information gain is
\begin{equation}
  \gamma_T=\max_{A\subset\calX,|A|=T}\tfrac12\log\det\bigl(\bI+\sigma_\varepsilon^{-2}\bK_A\bigr)
  \label{eq:infogain}
\end{equation}
where $\bK_A$ is the Gram matrix $\bK_t$ of \S~\ref{sec:bo_prelim} restricted to the rows and columns indexed by $A$. It
depends on the kernel's spectral decay \cite{srinivas2012gaussian}, and we
write $\tilde{O}(\cdot)$ for rates suppressing polylogarithmic factors in
$T$. Using Lemma~\ref{lem:tail}, define the
confidence width
\begin{equation}
  \rho_t = \Lambda_t^{1/q}, \qquad
  \Lambda_t = 2\log\!\Bigl(\frac{|\calX|\,t^2\pi^2}{6\delta}\Bigr),
  \label{eq:rho}
\end{equation}
and the \qucb acquisition function
\begin{equation}
  \alpha_{\qucb}(\bx\mid\calD_{t-1}) = \mu_{t-1}(\bx) + \rho_t\,\sigma_{t-1}(\bx).
  \label{eq:qucb}
\end{equation}
The second term uses $\sigma_{t-1}(\bx)$ because
Lemma~\ref{lem:tail} bounds $|U-m|$ in units of $\sigma$, the second
argument of the \qexpd{}. The exponent $1/q$ in place of the Gaussian $1/2$ is what distinguishes
$\rho_t$ from the usual GP-UCB schedule $\beta_t^{1/2}$ with
$\beta_t=\Lambda_t$ \cite{srinivas2012gaussian}, coinciding at $q{=}2$.

\begin{theorem}[Regret bound for \qucb]
\label{thm:regret}
Under (\textbf{A1}) and (\textbf{A2}), running \qucb with
$\rho_t$ as in Eq.~\eqref{eq:rho}, with probability at least $1-\delta$,
\begin{align}
  R_T &= \sum_{t=1}^T\bigl(f(\bx^*)-f(\bx_t)\bigr)
  \;\le\; 2\Lambda_T^{1/q}\sqrt{C_1\,T\,\gamma_T} \notag\\
  &= \tilde{O}\Bigl((\log T)^{1/q}\sqrt{T\gamma_T}\Bigr),
  \label{eq:regretbound}
\end{align}
where $C_1=8/\log(1+\sigma_\varepsilon^{-2})$ is the information-gain constant of \cite{srinivas2012gaussian} and $\gamma_T$ is the maximum
information gain of kernel $k$ as in Eq.~\eqref{eq:infogain}. Consequently,
$R_T/T\to0$ whenever $(\log T)^{2/q}\gamma_T=o(T)$. In particular, this holds under (\textbf{A1}) for every fixed $q>0$.
\end{theorem}
\begin{proof}[Proof sketch]
Union-bounding Lemma~\ref{lem:tail} over $\bx\in\calX$ and $t\ge1$
with $\delta_t=6\delta/(\pi^2t^2)$ yields valid confidence sets with
probability $\ge1-\delta$, and the standard UCB argument then gives
$r_t\le2\rho_t\sigma_{t-1}(\bx_t)$. Since $\sigma_{t-1}$ is the
posterior standard deviation in Eq.~\eqref{eq:qexpdbo_post}, the information-gain bound
$\sum_t\sigma_{t-1}^2(\bx_t)\le C_1\gamma_T$ applies, and with $\rho_t$
nondecreasing, the Cauchy--Schwarz inequality yields
$R_T\le2\rho_T\sqrt{C_1T\gamma_T}$ with
$\rho_T=\Lambda_T^{1/q}=O((\log T)^{1/q})$.
\end{proof}

\begin{corollary}[Cost of heavier predictive tails]
\label{cor:cost}
At $q{=}2$, $\rho_t=\Lambda_t^{1/2}$ and Eq.~\eqref{eq:regretbound} reduces to the classical rate $R_T=\tilde{O}(\sqrt{T\gamma_T\log T})$ \cite{srinivas2012gaussian}. For $q<2$, the confidence width $\Lambda_t^{1/q}$ in Eq.~\eqref{eq:rho} is wider by factor $\Lambda_t^{1/q-1/2}$, and the regret bound inherits exactly this factor while retaining the leading $\sqrt{T\gamma_T}$ term. Thus heavier predictive tails have a logarithmic exploration cost and preserve no-regret behavior whenever $(\log T)^{2/q}\gamma_T=o(T)$.
\end{corollary}

\subsubsection{\texorpdfstring{$q$}{q}-Exponential Expected Improvement (\qei)}
\label{sec:qei}

Expected improvement (EI) \cite{jones1998efficient} evaluates
the expected positive improvement over an incumbent
$f^+=\max_i y_i$. We compute this expectation under the
univariate \qexpd{} predictive in Definition~\ref{def:predictive}.
For $\sigma_t(\bx)>0$, define
\[
  d=\frac{\mu_t(\bx)-f^+}{\sigma_t(\bx)}.
\]
The standardized predictive variable
$Z=(f(\bx)-\mu_t(\bx))/\sigma_t(\bx)$ has CDF
\[
  F_q(z)=\Phi\!\left(\operatorname{sgn}(z)|z|^{q/2}\right),
\]
where $\Phi$ is the standard Gaussian CDF and
$\operatorname{sgn}(0)=0$. Define its upper partial first moment
for $a\geq0$ as
\[
  M_q(a)=\frac{2^{1/q-1}}{\sqrt{\pi}}\,
  \Gamma\!\left(\frac12+\frac1q,\frac{a^q}{2}\right),
\]
where $\Gamma(s,x)=\int_x^\infty v^{s-1}e^{-v}\,dv$
is the upper incomplete gamma function. The acquisition is
\begin{equation}
  \begin{aligned}
    \alpha_{\qei}(\bx)
      &= \mathrm{E}\!\left[(f(\bx)-f^+)_+\mid\calD_t\right]\\
      &= \sigma_t(\bx)\bigl[dF_q(d)+M_q(|d|)\bigr].
  \end{aligned}
  \label{eq:qei}
\end{equation}
To obtain Eq.~\eqref{eq:qei}, symmetry gives
$\mathrm{E}[(d+Z)_+]=dF_q(d)+\int_{|d|}^\infty z p_q(z)\,dz$,
where $p_q$ is the density of $Z$. Substituting $v=z^q/2$
evaluates the remaining integral as $M_q(|d|)$.
Thus, \qei{} has an exact expression in standard special
functions that incorporates the predictive shape as well as
its scale. At $q=2$, $F_2(d)=\Phi(d)$ and
$M_2(|d|)=\phi(d)$, recovering Gaussian EI, where $\phi$
is the standard Gaussian density. For $\sigma_t(\bx)=0$,
we use $\alpha_{\qei}(\bx)=(\mu_t(\bx)-f^+)_+$.

Unlike UCB, EI admits no simple confidence-interval-based regret
argument even in the Gaussian case, and extending existing EI consistency
results \cite{bull2011convergence} to the \qexpd{} predictive is left for
future work. This acquisition changes how predictive uncertainty
is scored, while retaining the GP mean and observed incumbent.
Its robustness to corrupted observations therefore requires
empirical evaluation.

\begin{algorithm}[t]
\DontPrintSemicolon
\caption{\qexpdbo}
\label{alg:qexpdbo}
\KwIn{Budget $T$, kernel $k$, tail parameter $q$}
\KwOut{Best observed point $\bx^*$, value $y^*$}
Evaluate $n_0$ initial random points and form $\calD_0$\;
\For{$t=1$ \KwTo $T$}{
  Fit kernel hyperparameters on $\calD_{t-1}$ by maximizing
  Eq.~\eqref{eq:qed} ($q$ fixed) and compute
  $\mu_{t-1}(\cdot),\sigma_{t-1}^2(\cdot)$ via
  Eqs.~\eqref{eq:gp_mean}--\eqref{eq:gp_var}\;
  Select $\bx_t$ by maximizing \qucb (with $\rho_t$ from Eq.~\eqref{eq:rho})
  or \qei, via Eqs.~\eqref{eq:qucb}/\eqref{eq:qei}\;
  Evaluate $y_t=f(\bx_t)+\varepsilon_t$ and update $\calD_t=\calD_{t-1}\cup\{(\bx_t,y_t)\}$\;
}
\Return $\bx^*=\argmax_{(\bx_i,y_i)\in\calD_T}y_i$
\end{algorithm}

\section{Experiments}
\label{sec:experiments}

We evaluate \qexpdbo{} on two signal-processing applications with expensive and occasionally corrupted evaluations: robust beamformer tuning (spatial filtering) and adaptive filter tuning (temporal filtering). In both cases the black-box maps design parameters to SINR or misalignment under corrupted covariance estimation or adaptation noise and has no analytic form. We report the best true SINR for beamforming and misalignment reduction for leaky NLMS, focusing on robust black-box tuning rather than on robustifying the beamformer or filter themselves.

\subsection{Beamformer Tuning with Impulsive Outliers}
\label{sec:exp-beamforming}

We tune a robust minimum variance distortionless response (MVDR) beamformer
\cite{capon1969highresolution}, treating output SINR as a black-box objective
over four interacting design parameters: diagonal loading
\cite{carlson1988covariance,li2003robust}, look direction, RCB uncertainty radius
\cite{vorobyov2003worst}, and covariance shrinkage \cite{ledoit2004well}.
We consider an $M{=}16$ ULA with
$N{=}24$ snapshots, SOI at $\theta_s{=}0^\circ$ (SNR~$=10\,\mathrm{dB}$), and
$K{=}4$ interferers at $\{-35^\circ,25^\circ,55^\circ,4^\circ\}$
(INR~$=25\,\mathrm{dB}$), including a near-mainlobe jammer, with search ranges
$\log_{10}(\lambda/\sigma_e^2)\in[-4,1]$,
$\theta_{\mathrm{look}}\in[\theta_s\pm5^\circ]$,
$\epsilon_{\mathrm{RCB}}\in[0,0.5]$, and $\alpha\in[0,1]$. We compare
\emph{Clean} Gaussian snapshots with an \emph{Impulsive} regime in which burst
and Cauchy outliers (peak rate $\pi{=}0.10$ near the interferer DOAs, floor
$0.05$) corrupt both the sample covariance and the online SINR meter
\cite{cox1987robust}, together with meter noise
($\sigma{=}0.35\,\mathrm{dB}$) and occasional Cauchy spikes.

All methods use $n_0{=}5$ Sobol starts, $T{=}100$ queries, squared exponential kernel, L-BFGS acquisition maximization, and 10 seeds. Kernel hyperparameters are refit at each iteration by maximizing Eq.~\eqref{eq:qed} for \qexpdbo{}, a Gaussian likelihood for GP baselines, and a Student-$t$ likelihood for TP-EI, using zeroth-order adaptive perturbation (ZAP) \cite{suwandi2026breaking}, a gradient-free optimizer that mitigates the curse of dimensionality in hyperparameter training. Baselines are GP-UCB, GP-EI \cite{srinivas2012gaussian,jones1998efficient}, TP-EI ($\nu{=}10$) \cite{shah2014student}, and truncated-mean GP-UCB (TruncMean) \cite{chowdhury2019heavytailed}, each compared against \qei/\qucb at $q\in\{1.5,1.0,0.5\}$. 

Figure~\ref{fig:mvdr_curves} shows the best-achieved SINR over time, and
Table~\ref{tab:mvdr} reports the final values. Pairwise comparisons use a paired
Wilcoxon signed-rank test across seeds ($p<0.05$), under which the $q{=}1.5$
gains of \qexpdbo{} over the strongest baseline are significant in the
impulsive regime of both tasks.

\textbf{Clean.} With Gaussian snapshots the GP model is well specified, and the
methods largely agree. GP-UCB, TP-EI, TruncMean, and \qucb at $q{=}1.5$
finish within ${\approx}0.2\,\mathrm{dB}$ of each other, and exact \qei{}
at $q{=}1$ is best overall, so a mildly heavy-tailed surrogate is
essentially free when outliers are absent. At $q{=}0.5$, \qucb{}'s SINR
drops by ${\sim}1\,\mathrm{dB}$ against GP-UCB, as wider UCB intervals
cause excessive exploration (Corollary~\ref{cor:cost}), though exact
\qei{} at $q{=}0.5$ closes this gap.

\textbf{Impulsive.} Once snapshots and the SINR meter are corrupted, \qucb and
\qei at $q{=}1.5$ outperform TruncMean, the strongest baseline, by
${\approx}0.7\,\mathrm{dB}$, while GP-UCB and TP-EI degrade. The matched-$q$
score downweights isolated spikes when fitting $\theta$, and the
heavier-tailed predictive widens exploration so the search does not
lock onto a corrupted region. TruncMean only truncates observations
and keeps a Gaussian surrogate on $f$, so it robustifies the fit but
not the exploration. \qexpdbo remains
strong at $q{=}1$, whereas $q{=}0.5$ again over-explores.

\begin{figure}[t]
\centering
\includegraphics[width=\columnwidth]{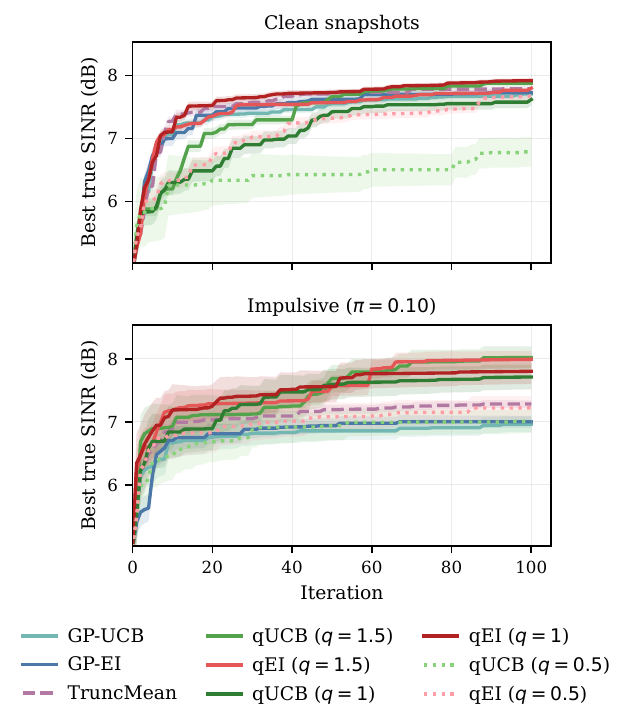}
\caption{Best true output SINR versus BO iteration (mean $\pm$ standard
error over 10 seeds). Top: clean Gaussian snapshots. Bottom: impulsive
corruption of the snapshots and the SINR meter.}
\label{fig:mvdr_curves}
\end{figure}
\begin{table}[t]
\centering
\caption{Final objective value (mean $\pm$ standard error over 10 seeds, higher
is better). Left: MVDR output SINR in dB. Right: NLMS misalignment reduction in
dB. Best entry in each column is bold.}
\label{tab:mvdr}
\setlength{\tabcolsep}{2.6pt}
\small
\begin{tabular}{lcccc}
\toprule
& \multicolumn{2}{c}{MVDR SINR} & \multicolumn{2}{c}{NLMS misalign.} \\
\cmidrule(lr){2-3}\cmidrule(lr){4-5}
Method & Clean & Impulsive & Clean & Impulsive \\
\midrule
GP-UCB & 7.71$\pm$0.12 & 6.96$\pm$0.37 & 30.4$\pm$0.22 & 13.9$\pm$0.61 \\
GP-EI & 7.73$\pm$0.17 & 7.00$\pm$0.28 & 30.5$\pm$0.19 & 14.2$\pm$0.52 \\
TP-EI & 7.71$\pm$0.17 & 6.87$\pm$0.32 & 30.2$\pm$0.26 & 13.6$\pm$0.66 \\
TruncMean & 7.79$\pm$0.11 & 7.28$\pm$0.18 & 30.3$\pm$0.21 & 14.7$\pm$0.38 \\
\qucb ($q{=}1.5$) & 7.88$\pm$0.15 & \textbf{8.02$\pm$0.53} & \textbf{30.6$\pm$0.18} & \textbf{15.9$\pm$0.44} \\
\qucb ($q{=}1$) & 7.62$\pm$0.28 & 7.71$\pm$0.59 & 30.1$\pm$0.31 & 15.4$\pm$0.57 \\
\qucb ($q{=}0.5$) & 6.79$\pm$0.70 & 7.01$\pm$0.56 & 28.4$\pm$0.74 & 13.8$\pm$0.83 \\
\qei ($q{=}1.5$) & 7.80$\pm$0.10 & 7.99$\pm$0.40 & 30.5$\pm$0.14 & 15.8$\pm$0.38 \\
\qei ($q{=}1$) & \textbf{7.92$\pm$0.11} & 7.80$\pm$0.60 & 30.4$\pm$0.20 & 15.4$\pm$0.55 \\
\qei ($q{=}0.5$) & 7.68$\pm$0.20 & 7.22$\pm$0.55 & 29.9$\pm$0.35 & 14.3$\pm$0.55 \\
\bottomrule
\end{tabular}
\end{table}

\subsection{Adaptive Filter Tuning with Impulsive Outliers}
\label{sec:exp-nlms}

We also consider tuning a leaky NLMS filter for system
identification \cite{mayyas1997leaky}. A length-$16$ plant is driven by white
Gaussian input, and the desired signal is observed through
Gaussian (\emph{Clean}) or $\epsilon$-contaminated Gaussian/Cauchy
(\emph{Impulsive}, $\pi{=}0.10$) noise. Impulsive outliers are the canonical
failure mode of LMS-family filters \cite{zou2000impulse}, since a single spike enters the tap
update directly. The black box is
$3$-dimensional, with $\log_{10}\mu\in[-3,0.3]$, $\log_{10}\delta\in[-6,0]$, and
leakage $\gamma\in[0,0.05]$, and it returns the steady-state misalignment
reduction $-10\log_{10}(\|\hat{\bw}-\bw^\star\|^2/\|\bw^\star\|^2)$ in dB, read
from a single run and corrupted by meter noise and occasional Cauchy spikes.

Table~\ref{tab:mvdr} (right) reports the final misalignment reduction.
The same Wilcoxon test as in \S~\ref{sec:exp-beamforming} finds the
$q{=}1.5$ gains statistically significant in the impulsive regime. 

\textbf{Clean.} With Gaussian observation noise the surrogates are hard
to separate: every method except \qucb{} at $q{=}0.5$ finishes within
${\approx}0.5\,\mathrm{dB}$ of the others and within ${\approx}1\,\mathrm{dB}$
of the $31\,\mathrm{dB}$ grid optimum. \qucb{} at $q{=}1.5$ is again
slightly ahead, and \qei{} at $q{=}1.5$ matches GP-EI to the reported
precision.

\textbf{Impulsive.} Under $\epsilon$-contaminated noise, all methods lose roughly half the attainable misalignment reduction, but
\qucb{} and \qei{} at $q{=}1.5$ recover $1.1$--$1.2\,\mathrm{dB}$ more
than TruncMean, the strongest baseline, and $q{=}1$ still finishes ahead
of it. \qei{} at $q{=}1.5$ finishes within $0.2\,\mathrm{dB}$ of
\qucb{}. As in the beamformer task, the matched-$q$
score keeps the acquisition from locking onto a corrupted region of $(\mu,\delta,\gamma)$.
\section{Conclusion}
\label{sec:conclusion}
We introduced \qexpdbo{}, whose univariate $q$-exponential predictive recovers the GP at $q=2$ and grows heavier tailed as $q$ decreases. From its stretched exponential tail bound, we prove a sublinear regret guarantee for \qucb{}, and we derive exact \qei{} under the same predictive via the Gaussian CDF and the upper incomplete gamma function. On beamformer and leaky NLMS tuning, $q\in[1,1.5]$ matches GP baselines on clean data ($q{=}1$ best for \qei{}) and outperforms the strongest baseline under impulsive outliers, while $q{=}0.5$ overexplores, so we recommend $q{=}1.5$ as a safe default and a smaller $q$ for heavily corrupted settings. Because the predictive modifies only the tail, \qexpdbo{} remains compatible with any kernel, including adaptively designed ones \cite{suwandi2025adaptive}. We fix $q$ before optimization and leave its joint or online learning \cite{zhang2023dataadaptive} to future work.

\bibliographystyle{IEEEbib}
\bibliography{refs}

\end{document}